\documentclass[dvipsnames,format=sigconf,anonymous=false,review=false]{acmart}
\AtBeginDocument{%
  }

\copyrightyear{2026}
\acmYear{2026}
\setcopyright{cc}
\setcctype{by}
\acmConference[GECCO Companion '26]{Genetic and Evolutionary Computation Conference}{July 13--17, 2026}{San Jose, Costa Rica}
\acmBooktitle{Genetic and Evolutionary Computation Conference (GECCO Companion '26), July 13--17, 2026, San Jose, Costa Rica}
\acmDOI{10.1145/3795101.3805359}
\acmISBN{979-8-4007-2488-6/2026/07}

\usepackage{graphicx}
\usepackage[vlined,linesnumbered, ruled]{algorithm2e}
\usepackage[createShortEnv]{proof-at-the-end}
\usepackage{cleveref}
\usepackage{amsthm}
\SetKwRepeat{Do}{do}{while}%
\SetKwIF{WithProb}{ElseWithProb}{Else}{with probability}{do}{else with probability}{else}{end}
\newcommand{\assign}{\leftarrow}

\newcommand{\N}{\mathbb{N}}
\newcommand{\natnum}{\N}

\newcommand{\ExMath}[1]{\mathrm{E}\left[#1\right]}
\newcommand{\PrMath}[1]{\mathrm{Pr}\left[#1\right]}
\newcommand{\BigO}[1]{\mathcal{O}\left(#1\right)}

\newcommand{\firstInts}[1]{[1..#1]}

\newcommand{\OMath}[1]{\BigO{#1}}

\newcommand{\Ex}[1]{\ExMath{#1}}
\newcommand{\Prob}[1]{\PrMath{#1}}

\newcommand{\Brackets}[1]{\left ( #1 \right )}

\newcommand{\oneoneEA}{(1+1)~EA\xspace}

\definecolor{mygreen}{RGB}{1, 150, 122}

\providecommand{\ignore}[1]{}

\renewcommand{\epsilon}{\varepsilon}

\begin{document}

\title{Analysis of Search Heuristics in the Multi-Armed Bandit Setting}

\author{Jasmin Brandt}
\email{jbrandt@techfak.uni-bielefeld.de}
\orcid{0000-0002-6364-2081}
\affiliation{%
  \institution{University of Bielefeld}
  \city{Bielefeld}
  \country{Germany}
}

\author{Barbara Hammer}
\email{bhammer@techfak.uni-bielefeld.de}
\orcid{0000-0002-0935-5591}
\affiliation{%
  \institution{University of Bielefeld}
  \city{Bielefeld}
  \country{Germany}
}

\author{Timo Kötzing}
\email{timo.koetzing@hpi.de}
\orcid{0000-0002-1028-5228}
\affiliation{%
  \institution{Hasso Plattner Institute\\ University of Potsdam}
  \city{Potsdam}
  \country{Germany}
}

\author{Jurek Sander}
\email{Jurek.Sander@hpi.de}
\orcid{0009-0004-6707-6592}
\affiliation{%
  \institution{Hasso Plattner Institute\\ University of Potsdam}
  \city{Potsdam}
  \country{Germany}
}

\renewcommand{\shortauthors}{xxx}

\begin{abstract}
    We consider the classic Multi-Armed Bandit setting to understand the exploration/exploitation tradeoffs made by different search heuristics. Since many search heuristics work by comparing different options (in evolutionary algorithms called ``individuals''; in the Bandit literature called ``arms''), we work with the ``Dueling Bandits'' setting. In each iteration, a comparison between different arms can be made; in the binary stochastic setting, each arm has a fixed winning probability against any other arm. A Condorcet winner is any arm that beats every other arm with a probability strictly higher than $1/2$.

    We show that evolutionary algorithms are rather bad at identifying the Condorcet winner: Even if the Condorcet winner beats every other arm with a probability $1-p$, the \oneoneEA, in its stationary distribution, chooses the Condorcet winner only with constant probability if $p=\Omega(1/n)$. By contrast, we show that a simple EDA (based on the Max-Min Ant System with iteration-best update) will choose the Condorcet winner in its maintained distribution with probability $1-\Theta(p)$. As a remedy for the \oneoneEA, we show how repeated duels can significantly boost the probability of the Condorcet winner in the stationary distribution.
\end{abstract}



\maketitle

\section{Introduction}

Randomized search heuristics, such as evolutionary algorithms (EAs) or estimation of distribution algorithms (EDAs) work iteratively to find better and better search points until, hopefully, finding a global optimum. Key to this iterative search is \emph{exploiting} the knowledge gained about the search while being open to \emph{explore} alternative options. This poses the well-known \emph{exploration–exploitation dilemma}, also known as the \emph{exploration–exploi\-ta\-tion tradeoff}.

A standard way to model sequential decision making and for analyzing algorithms with respect to this tradeoff is the \emph{Multi-Armed Bandits} model. This model considers $n$ different given options, called \emph{arms} (stemming from the image of $n$ slot machines, single-armed bandits, to choose from). Immediately after choosing an arm a (potential noisy) reward can be observed, usually sampled according to an unknown distribution. The goal is then to identify the arm with the highest expected reward in as few time steps as possible, but with high confidence. Note that, in the literature on randomized search heuristics, the arms would be called ``individuals'' or ``search points''. 

A commonly known and well-investigated area of Multi-Armed Bandits is the preference-based Bandit variant, also known as (Multi\nobreakdash-) Dueling \cite{Brost2016}, Battling \cite{Saha2018}, Choice \cite{Agarwal2020}, or Combinatorial Bandits~\cite{Brandt2022}. While in classical Multi-Armed Bandits the agent can only choose one arm per time step, in Multi-Dueling Bandits it is possible to select a whole set of $k \in \mathbb{N}$ out of all possible $n \in \mathbb{N}, n \geq k$ arms. The feedback that can be observed after pulling such a set of arms usually only contains preference-based information about the ``winning'', or, in other words, the best arm contained in the selected subset. For a more detailed survey on preference-based Bandit feedback, see \cite{BengsDuelingBanditsSurvey}. Note that the observation usually is probabilistic and sampled according to a fixed pairwise or, respectively, setwise winning probability for each subset of arms.

A possible way to define such setwise winning probabilities is with a statistical model like the Bradley-Terry or its generalization, the Plackett-Luce model \cite{Luce1959, Plackett1975}. They assume an underlying utility for each of the arms from which the probability of a particular ranking can be computed: in the resulting ranking, each arm wins with probability proportional to its
utility.

In the Bandit literature, multiple concepts of an "optimal arm" exist, see \cite{BengsDuelingBanditsSurvey}. However, the arm with the maximal utility in an underlying Plackett-Luce model is automatically the most probable winner in each subset it is contained in. Such an arm wins in expectation in each subset it is contained in, referred to as the \textit{Condorcet winner} in the existing literature. A huge body of literature exists containing algorithms that are specifically designed to solve the best arm identification problem in Multi-Dueling Bandits, eg.~\cite{Mohajer2017, Wenbo2019, Wenbo2020, Chen2020, Saha2020, Haddenhorst2021}. However, to apply existing and theoretically grounded algorithms from other research fields was not tested until now.

In this paper, we consider the version of \emph{(Multi-) Dueling Bandits}: A set of arms, typically two, are compared, with the feedback to the algorithm being which arm won the comparison. The simplest case is that of deterministic feedback, where any comparison of two arms gives a deterministic winner. While there need not be an arm that wins against all other arms, we briefly study the setting where there is such a winner and show that a simple randomized search heuristic, based on the \oneoneEA, efficiently finds a winner in an expected number of $\BigO{n}$ queries, as discussed in Section~\ref{sec:deterministicQueries}.

More interesting is the case of stochastic queries: instead of getting deterministic feedback, for each pair of arms there is an unchanging probability that the first arm wins the comparison.

In Section~\ref{sect:unst-cond-win-stoch} we consider a simple variant of the well-known \oneoneEA: While maintaining a ``current winner'', we sample a random arm to pit the current winner against. The winner of this comparison will be our new current winner. This process defines a Markov chain on the different arms, of which we are interested whether the stationary distribution identifies an arm that beats every other arm with probability more than 50\% (the Condorcet winner for this setting). We consider the case where a Condorcet winner exists and want the algorithm to be able to identify it in the stationary distribution.

We well established in Corollary~\ref{cor:oneoneEA_Stationary} that in order for the stationary distribution to identify the Condorcet winner with probability~$1~-~o(1)$, it needs to beat all other arms with probability $1-o(1/n)$. In Lemma~\ref{lem:oneonEA_stochasticQueries} we give conclusions for general bounds. Additionally, we show fast mixing of the induces Markov chain in Lemma~\ref{lem:oneoneEA_mixingTime}. 

While these results hold for general distributions, in Section~\ref{sec:PLBoost} we consider specifically the Plackett-Luce Model. Since our previous results show that only very strong signals can be picked up by evolutionary algorithms, we consider boosting the signal by performing multiple duels (on the same set) before making a decision. In Corollary~\ref{cor:plackett_luce} we see how we can boost a Condorcet winner which wins against any arm with probability at least $1/2+\delta$ to appear in the stationary distribution with probability $1-\Theta(1/n)$ by making $\BigO{(1/\delta)^2\ln(n)}$ duels per iteration.

In Section~\ref{sec:aco_analysis}, we consider a simple estimation of distribution algorithm (EDA) based on the Max-Min Ant System with iteration best update (see~\cite{DBLP:series/ncs/KrejcaW20} for a discussion). This algorithm maintains a list of marginal probabilities, one for each arm. In each iteration, two arms are sampled according to the distribution defined by this list of probabilities, and a single duel is performed. We show, in Section~\ref{sec:aco_analysis}, that this EDA can identify a Condorcet winner much better: The marginal probability of the Condorcet winner converges quickly to $1-\Theta(p)$, if the Condorcet winner beats every other arm with probability at least $1-p$.

Before giving the details of our results, we discuss tools and methods in Section~\ref{sect:prelim}.
Note that, for better readability, most of the proofs of Section~\ref{sec:PLBoost} are not in the main part of the paper, but can be found in the appendix.

\section{Methods and Tools}
\label{sect:prelim}

For any natural number $n \in \natnum$, we use $\firstInts{n}$ to denote the set $\{1,\ldots,n\}$. More generally, for $n,m \in \natnum$ with $n \leq m$, we let $[n..m] = \{n,\ldots,m\}$.

The most common technique in the analysis of randomized search
heuristics is Drift Analysis (see~\cite{kötzing2024theorystochasticdrift} for an overview). While we use
these techniques for the analysis of the EDA in Section~\ref{sec:aco_analysis}, our other work employs mainly Markov chain analysis, in particular the stationary distribution and mixing times (see~\cite{Mitzenmacher2017ProbComp} for an overview). This technique is occasionally used in the analysis of randomized search heuristics, for example \cite{Rudolph1998FiniteMarkovChainResults,KötzingAnalysisOfEAOnLeadingOnesWithConstraints,Jägersküpper2008MarkovChainAndDriftAnalysis}.

Of particular interest to us are coupling techniques, see~\cite{Doerr2020TheoryOfEvoComp} for a general introduction in the context of evolutionary algorithms. Also this technique has been used occasionally, first in \cite{Witt2006RuntimeAnalysis(μ+1)EA} and~\cite{WITT2008PopulationSizeVersusRuntime} for the analysis of the $(\mu +1)$~EA and an elitist steady-state genetic algorithm, later for memetic algorithms \cite{SudholtImpactOfParametrizationInMemeticEAs}, aging mechanisms~\cite{JANSEN2011RandomizedSearchHeuristics}, non-elitist algorithms \cite{LehreMutationSelectionBalance}, multi-objective evolutionary algorithms~\cite{Doerr2013MultiObjectiveEA}, the $(\mu+\lambda)$~EA \cite{AntipovRuntimeAnalysisEA}, and ant colony optimization \cite{Sudholt2011MarkovChainMixingTimeEstimates}.

We use the following definition of coupling, as given by \cite[Definition~11.3]{Mitzenmacher2017ProbComp}.

\begin{definition}\label{def:prelim:coupling}
    A coupling of a Markov chain $M_t$ with state space~$S$ is a Markov chain $Z_t = (X_t, Y_t)$ on the state space $S \times S$ such that:
    \begin{align*}
        \PrMath{X_{t+1} = x' \mid Z_t = (x,y)} = \PrMath{M_{t+1} = x' \mid M_t = x}\\
        \PrMath{Y_{t+1} = y' \mid Z_t = (x,y)} = \PrMath{M_{t+1} = y' \mid M_t = y}.
    \end{align*}
\end{definition}

We use the following lemma, as proven in \cite[Lemma 11.2]{Mitzenmacher2017ProbComp}, to get the mixing time $\tau(\epsilon)$ of Markov chains until the Markov chain is "close" to its steady state distribution. 

\begin{lemma}
    Let $\epsilon < 1$ and $Z_t = (X_t, Y_t)$ be a coupling of a Markov chain on a state space $S$. 
    Suppose that there exists a $T\in \natnum$ such that, for every $x, y \in S$
    \begin{align*}
        \PrMath{X_T \neq Y_T \mid X_0 = x, Y_0 = y} \leq \epsilon.
    \end{align*}
    Then the mixing time is $\tau(\epsilon) \leq T$.
    \label{lem:prelim:mixing-time}
\end{lemma}

Given the mixing time and stationary distribution of a Markov chain, we use the next Theorem, proven by Sudholt in \cite{Sudholt2011MarkovChainMixingTimeEstimates}, to determine the expected optimization time.

\begin{theorem}
    Consider a randomized search heuristic that can be represented by an ergodic Markov chain with a stationary distribution~$\pi$.
    Let OPT be the set of global optima, and let $\tau(\epsilon)$ denote the mixing time on the considered problem.
    Then, the expected optimization time is at most
     $   \tau(\epsilon) \cdot \OMath{\log(1/\pi(\text{OPT}))}/\pi(\text{OPT}).
    $
    \label{thm:prelim:exp-opt-time}
\end{theorem}
\section{Deterministic Queries}
\label{sec:deterministicQueries}

Suppose we have $n$ arms to choose from. The arms do not have a quality value as such, but, for a given $k \in [2..n]$, we can compare any choice of $k$ arms and (reliably) obtain the best of the arms. The task is now to find the best of all $n$ arms with as few comparisons as possible.

In the Bandit literature this setting of multiple competitors in each time step is known as the \textit{(Multi-) Dueling Bandit} setting, like in \cite{Brost2016MultiDueling}. However, in contrast to the classical Bandit feedback, which is sampled by an underlying stochastic process, we assume a deterministic feedback here.

\begin{definition}[Winner Search, comparison-based, deterministic]
    Given $k,n \in \natnum$ and a winner $i^* \in [1..n]$. An algorithm can perform a query on a chosen $A \subseteq [1..n]$ with $|A|\leq k$ deterministically returning some element $q(A) \in A$. We assume that, for all $A~\subseteq~[1..n]$ with $|A| \leq k$ and $i^* \in A$, we have $q(A) = i^*$. Intuitively, $i^*$ always wins. The goal is to identify $i^*$.
\end{definition}

Our first algorithm is the Round-Robin algorithm. It compares all options in turn against the current winner.

\begin{algorithm}[htp]
    \textbf{Parameter:} $n \in \natnum$\;
    Let $i^*$ $\assign$ $1$\;
    \For{$i=2$ \KwTo $n$}{
        $i^*$ $\assign$ Query$(i^*,i)$\;
    }
    \caption{Round-Robin}
\end{algorithm}

\begin{theorem}
    The Round-Robin algorithm identifies a winner within $n-1$ queries of size~$2$.
\end{theorem}

While the Round-Robin algorithm serves as a base-line, we are interested in how randomized search heuristics fare in such settings. The next algorithm we present is the simplest randomized search heuristic for our unstructured search space.
We use, for a finite set~$S$, $U(S)$ as the uniform distribution on $S$.

\begin{algorithm}[htp]
    \textbf{Parameter:} $n \in \natnum$\;
    $i^*$ $\assign$ $1$\;
    \For{$t=0$ \KwTo $\infty$}{
        $i$ $\assign$ $U([1..n])$\;
        $i^*$ $\assign$ Query$(i^*,i)$\;
    }
    \caption{Random Search, non-structured state space}
\end{algorithm}

\begin{theorem}
    The Random Search algorithm performs queries of size $2$ and maintains the winner as $i^*$ after an expected number of $n$ iterations. The distribution of the number of iterations until maintaining the winner as $i^*$ follows $\mathrm{Geo}(1/n)$.
\end{theorem}

\begin{proof}
    Since the algorithm selects the winner, in each iteration, with probability~$1/n$, the number of iterations $T$ until it is found is geometrically distributed with parameter $p = 1/n$; $T \sim \mathrm{Geo}(p)$.
    
    Using $\Ex{T} = 1/p = n$, the expected number of iterations after which the algorithm maintains the winner $i^*$ follows.
\end{proof}
\section{Stochastic Queries: the \oneoneEA}
\label{sect:unst-cond-win-stoch}
Suppose we have $n$ arms to choose from and, for a given $k~\in~[2..n]$, we can query which of these arms is best; however, the answer is now a random element of the set. Any option that comes out as the winner against any other option with probability strictly more than~50\% is called a \emph{Condorcet winner} in the existing literature, see e.g. \cite{BengsDuelingBanditsSurvey}. If such a Condorcet winner exists, the task is now to find it with as few comparisons as possible.

\begin{definition}[Condorcet Winner Search]
    Given $n \in \natnum$ and a matrix $M$ with values in $(0,1)$ such that, for all $1 \leq i < j \leq n$, we have $M(i,j) + M(j,i) = 1$ and $M(i,i)=1$. An algorithm can perform a query on any two different arms $i,j \in \firstInts{n}$ with return value $i$ with probability~$M(i,j)$ and with return value $j$ otherwise. The goal is to identify $i^*\in \firstInts{n}$ such that, for all $i \in \firstInts{n} \setminus \{i^*\},$ $M(i^*,i) > 0.5$.
    \label{def_condorcet_winner_search}
\end{definition}

In order to find the Condorcet winner, we define a variant of the standard \oneoneEA that compares the current best solution, the \emph{incumbent}, with a randomly selected arm that acts as \emph{challenger} in each iteration.
Let $i \in \firstInts{n}$ be the incumbent and $j \in \firstInts{n}$ the challenger of the current iteration.
The query returns $i$ with probability $M(i,j)$ and $j$ with probability $M(j,i) = 1 - M(i,j)$ as defined for the stochastic Condorcet winner search.
Algorithm~\ref{alg_1p1EA_Condorcet} shows the \oneoneEA in the Condorcet winner search setting; we use, for a finite set $S$, $U(S)$ as the uniform distribution on $S$.

\begin{algorithm}[htp]
    \textbf{Parameter:} $n \in \natnum$\;
    $i_0^*$ $\assign$ $U([1..n])$\;
    \For{$t=0$ \KwTo $\infty$}{
        $i_t$ $\assign$ $U([1..n])$\;
        $i_t^*$ $\assign$ Query$(i_{t-1}^*,i_t)$\; 
        \tcc{returns $i_{t-1}^*$ with prob. $M(i_{t-1}^*,i_t)$ and $i_t$ with prob. $M(i_t, i_{t-1}^*)$}
    }
    \caption{\oneoneEA, Condorcet Winner Search}
    \label{alg_1p1EA_Condorcet}
\end{algorithm}

Like the standard \oneoneEA, this variant does not store any information about the results of previous iterations.
The current best-found solution $i_t^*\in \firstInts{n}$ depends only on the previous incumbent and the selected challenger in this iteration.
We observe $i^*_t$ as a \emph{finite} process, since it takes values from the finite set of $n$ arms.
The stochastic process $i^*_t$ is a Markov chain because each state $i^*_t$ depends only on the previous state $i^*_{t-1}$.
The transition probabilities for all~$i,j \in \firstInts{n}$ are defined by
\begin{align}
    P_{i,j} &= \Prob{i^*_t = j \mid i^*_{t-1} = i}\notag\\
    &= \begin{cases}
            \frac{1}{n} M(j,i),& \text{for $i \neq j$}; \\
            \frac{1}{n} \sum_{j \in \firstInts{n}} M(i,j), & \text{otherwise.}
        \end{cases}
    \label{form_trans_prob}
\end{align}

Because $P_{i,j}>0$ for all $i,j \in S$, the Markov chain $i^*_t$ is finite, irreducible, and aperiodic, and hence also ergodic (see \cite[Corollary~7.6]{Mitzenmacher2017ProbComp}).
Therefore, it has a unique stationary distribution $\overline{\pi} = (\pi_1, \pi_2, \ldots, \pi_n)$ (see \cite[Theorem~7.7]{Mitzenmacher2017ProbComp}).
In the following, we use \cite[Theorem~7.10]{Mitzenmacher2017ProbComp} to calculate the stationary distribution and bound it for the Condorcet winner.

\begin{lemmaE}[][normal]  \label{lem:oneonEA_stochasticQueries}
    Consider the \oneoneEA variant for the Condorcet winner search, $n \in \N$, and $S = \firstInts{n}$ be the state space.
    Let $i^*_t\in S$ be the underlying Markov chain for $t \in \natnum$ that denotes the current best solution in the $t$-th iteration of the algorithm, $\overline{\pi} = (\pi_1,\ldots,\pi_n)$ its unique stationary distribution and $i^* \in \N$ the unique Condorcet winner. Suppose that there are bounds $p_l,p_u$ such that, for all $i \in S$ with $i \neq i^*$,  
    $$
    1 > 1-p_u \geq M(i^*, i) \geq 1- p_l > \frac{1}{2}.
    $$
    Then the stationary distribution can be bounded by
    \begin{align*}
        \pi_{i^*} \geq \frac{1-p_l}{1-p_l+(n-1)\cdot p_l} > \frac{1-p_l}{1-p_l+n\cdot p_l}
    \end{align*}
    and
    \begin{align*}
        \pi_{i^*} \leq \frac{1-p_u}{1-p_u+(n-1)\cdot p_u} < \frac{1-p_u}{n\cdot p_u}.
    \end{align*}
    \label{lem:bin:stationary_dist_bounds}
\end{lemmaE}

\begin{proof}[Proof sketch]
    Using \cite[Theorem 7.10]{Mitzenmacher2017ProbComp}, the unique stationary distribution $\overline{\pi} = (\pi_1, \pi_2, ..., \pi_n)$ can be determined by solving
    \begin{align}
        \sum_{i \in \firstInts{n}} \pi_i & = 1 \label{eq:bin:stat_init_first_line-sketch}\\
        \pi_i &= \pi_j \frac{P_{j,i}}{P_{i,j}}, \quad \text{for every $i,j \in \firstInts{n}, i \neq j$} \label{eq:bin:stat_init_second_line-sketch}
    \end{align}
    with $P_{i,j}$ as defined in (\ref{form_trans_prob}).
    Solving (\ref{eq:bin:stat_init_second_line-sketch}) for all combinations of $\pi_{i^*}$ and  $i \in \firstInts{n}$ with $i\neq i^*$, and inserting it into (\ref{eq:bin:stat_init_first_line-sketch}), results in
    \begin{align*}
        \pi_{i^*} = \frac{1}{1 + \sum_{i \in [1..n], i\neq i^*} \frac{M(i,i^*)}{M(i^*,i)}}.
    \end{align*}
    Using the bounds of $M(i^*, i)$ concludes the proof of Lemma~\ref{lem:bin:stationary_dist_bounds}.
\end{proof}

\begin{proofE}
    Using \cite[Theorem 7.10]{Mitzenmacher2017ProbComp}, the unique stationary distribution $\overline{\pi} = (\pi_1, \pi_2, ..., \pi_n)$ can be determined by solving
    \begin{align}
        \sum_{i \in \firstInts{n}} \pi_i & = 1 \label{eq:bin:stat_init_first_line}\\
        \pi_i &= \pi_j \frac{P_{j,i}}{P_{i,j}}, \quad \text{for every $i,j \in \firstInts{n}, i \neq j$} \label{eq:bin:stat_init_second_line}
    \end{align}
    with $P_{i,j}$ as defined in (\ref{form_trans_prob}).
    For every $i \in \firstInts{n}$ with $i \neq i^*$, (\ref{eq:bin:stat_init_second_line}) gives
    \begin{align}
        \pi_{i} &= \pi_{i^*} \frac{P_{i^*,i}}{P_{i,i^*}}\notag\\
        &= \pi_{i^*} \frac{\frac{1}{n}M(i,i^*)}{\frac{1}{n}M(i^*,i)}\notag\\
        &= \pi_{i^*} \frac{M(i,i^*)}{M(i^*,i)}.\label{eq:bin:stat_distr_second_line_transformed}
    \end{align}
    Using (\ref{eq:bin:stat_distr_second_line_transformed}) in (\ref{eq:bin:stat_init_first_line}) leads to
    \begin{align}
        &\pi_{i^*} + \sum_{i \in \firstInts{n}, i\neq i^*}\pi_{i} = 1\notag\\
        \Leftrightarrow \space & \pi_{i^*} + \sum_{i \in [1..n], i\neq i^*}\pi_{i^*} \frac{M(i,i^*)}{M(i^*,i)} = 1\notag\\
        \Leftrightarrow \space & \pi_{i^*} = \frac{1}{1 + \sum_{i \in [1..n], i\neq i^*} \frac{M(i,i^*)}{M(i^*,i)}}. \label{eq:bin:stat_distr_cond_win}
    \end{align}
    
    In the following, we will conclude the proof of Lemma~\ref{lem:bin:stationary_dist_bounds} by inserting the boundaries of $M(i^*,i)$ in (\ref{eq:bin:stat_distr_cond_win}). Recall that $M(i^*,i) = 1-M(i,i^*)$.

    Since for all $i \in S$ with $i \neq i^*$ it holds that $M(i^*, i) \geq 1-p_l > \frac{1}{2}$, we get
    \begin{align}
        \frac{M(i,i^*)}{M(i^*,i)} \leq \frac{p_l}{1-p_l}.\label{eq:bin:lower-bound-frac}
    \end{align}
    Using (\ref{eq:bin:lower-bound-frac}) in (\ref{eq:bin:stat_distr_cond_win}), results in the lower bound
    \begin{align*}
        \pi_{i^*} &= \frac{1}{1 + \sum_{i \in [1..n], i\neq i^*} \frac{M(i,i^*)}{M(i^*,i)}}\\
        &\geq \frac{1}{1 + (n-1) \frac{p_l}{1-p_l}}\\
        &= \frac{1-p_l}{1-p_l+(n-1)\cdot p_l}\\
        &> \frac{1-p_l}{1-p_l+n\cdot p_l}.
    \end{align*}

    The upper bound can be calculated in a similar manner.
    For all~$i \in S$ with $i \neq i^*$, it holds that $1 > 1 - p_u \geq M(i^*, i)$, which results in
    \begin{align}
        \frac{M(i,i^*)}{M(i^*,i)} \geq \frac{p_u}{1-p_u}.\label{eq:bin:upper-bound-frac}
    \end{align}
    Using (\ref{eq:bin:upper-bound-frac}) in (\ref{eq:bin:stat_distr_cond_win}) and $1-p_u > p_u$ in the last inequality, leads to
    \begin{align*}
        \pi_{i^*} &= \frac{1}{1 + \sum_{i \in [1..n], i\neq i^*} \frac{M(i,i^*)}{M(i^*,i)}}\\
        &\leq \frac{1}{1 + (n-1) \frac{p_u}{1-p_u}}\\
        &= \frac{1-p_u}{1-p_u+(n-1)\cdot p_u}\\
        &< \frac{1-p_u}{n \cdot p_u}.
    \end{align*}
\end{proofE}

The detailed proof can be found in the appendix. Using Lemma~\ref{lem:bin:stationary_dist_bounds}, the stationary distribution can be calculated for a specific error term $\gamma$ in the following proposition.

\begin{propositionE}[][normal]
    Consider the \oneoneEA variant for the Condorcet winner search, $n \in \N$, and $S = \firstInts{n}$ be the state space.
    Let $i^*_t\in S$ be the underlying Markov chain for $t \in \natnum$ that denotes the current best solution in the $t$-th iteration of the algorithm, $\overline{\pi} = (\pi_1,...,\pi_n)$ its unique stationary distribution and $i^* \in \N$ the unique Condorcet winner.
    Let $0 < \gamma < 1$ and 
    \begin{align*}
        p = \frac{\gamma}{\gamma+(1-\gamma)(n-1)}.
    \end{align*}
    For all $i \in S$ with $i \neq i^*$, let $M(i^*, i) = 1-p$.
    Then, the stationary distribution of the Condorcet winner is
    \begin{align*}
        \pi_{i^*} = 1-\gamma.
    \end{align*}
    \label{prop:bin:stationary_dist_bounds}
\end{propositionE}

\begin{proofE}
    Using the similarity of the lower and upper bound in Lemma~\ref{lem:bin:stationary_dist_bounds}, inserting $p= \frac{\gamma}{\gamma+(1-\gamma)(n-1)}$ into both bounds of the Lemma results in
    \begin{align*}
        \pi_{i^*} &= \frac{1-p}{1-p + (n-1)p}\\
        &= \frac{1-\frac{\gamma}{\gamma+(1-\gamma)(n-1)}}{1-\frac{\gamma}{\gamma+(1-\gamma)(n-1)} + (n-1)\frac{\gamma}{\gamma+(1-\gamma)(n-1)}}\\
        &= \frac{\frac{\gamma+(1-\gamma)(n-1)-\gamma}{\gamma+(1-\gamma)(n-1)}}{\frac{\gamma+(1-\gamma)(n-1)-\gamma + (n-1)\gamma}{\gamma+(1-\gamma)(n-1)}}\\
        &= \frac{(1-\gamma)(n-1)}{(1-\gamma)(n-1) + (n-1)\gamma}\\
        &= \frac{1-\gamma}{1-\gamma+\gamma}\\
        &= 1-\gamma.
    \end{align*}
\end{proofE}

Applying Proposition~\ref{prop:bin:stationary_dist_bounds} leads to the following corollary.

\begin{corollary}\label{cor:oneoneEA_Stationary}
    Consider the \oneoneEA variant for the Condorcet winner search, $n \in \N$, and $S = [1..n]$ be the state space.
    Let $i^*_t\in S$ be the underlying Markov chain for $t \in \natnum$ that denotes the current best solution in the $t$-th iteration of the algorithm, $\overline{\pi} = (\pi_1,\ldots,\pi_n)$ its unique stationary distribution and $i^* \in \N$ the unique Condorcet winner.
    Then the stationary distribution of the Condorcet winner can be calculated for different values of $M(i^*, i)$ as follows.
    \begin{itemize}
        \item[1.] Let 
        $p_1 = o(1)/n$
        and for all $i \in S$ with $i \neq i^*$, let $M(i^*, i) = 1-p_1$.
        Then, the stationary distribution of the Condorcet winner is
        \begin{align*}
            \pi_{i^*,1} = 1-o(1).
        \end{align*}
        \item[2.] Let $p_2 = \Theta\Brackets{1/n^2}$
        and for all $i \in S$ with $i \neq i^*$, let $M(i^*, i) = 1-p_2$.
        Then, the stationary distribution of the Condorcet winner is
        \begin{align*}
            \pi_{i^*,2} = 1-\Theta(1/n).
        \end{align*}
    \end{itemize}
    If we additionally assume $M(i^*,i)=M(i^*,j)$ for all $i,j \in \firstInts{n}\setminus i^*$, for both parts of the corollary, the converse also holds.
    To get at least the specific stationary distribution $\pi_{i^*,1}$ or $\pi_{i^*,2}$, $M(i^*, i)$ must be at least $1-p_1$ or $1-p_2$ for all $i \in S$ with $i \neq i^*$.
    \label{cor:bin:stationary_dist_bounds}
\end{corollary}

After calculating the stationary distribution of the Condorcet winner $i^*$, we are interested in the expected number of iterations it takes the \oneoneEA variant in Algorithm~\ref{alg_1p1EA_Condorcet} to obtain it.
Intuitively, the mixing time $\tau(\epsilon)$ of a Markov chain describes the number of iterations until the distribution of the state of the chain is close to the stationary distribution.
The complete derivation of mixing times and how to get them using coupling can be found in the appendix.

Using Lemma~\ref{lem:prelim:mixing-time}, we get the mixing time for the described Markov chain of the \oneoneEA variant for the Condorcet winner search in the following lemma.

\begin{lemmaE}[][normal]\label{lem:oneoneEA_mixingTime}\label{lem_coupling_res}
    Consider the \oneoneEA variant for the Condorcet winner search and $n \in \N$.
    Let $i^*_t\in S$ be the underlying Markov chain for $t \in \natnum$ that denotes, for every iteration $t \in \natnum$, the current best solution in the $t$-th iteration of the algorithm.
    Then, $\tau(\epsilon) \leq n \ln (1/\epsilon).$
    
\end{lemmaE}

\begin{proofE}
    We define two copies $X_t$ and $Y_t$ of $i^*_t$ as follows.
    Let $i, j \in [1..n]$, $X_t = i$ and $Y_t = j$.
    Suppose $C_t$ is a stochastic process that describes, for every iteration $t \in \natnum$, the challenger $k \in [1..n]$.
    To obtain $X_{t+1}$ and $Y_{t+1}$ we choose a challenger $C_t = k \in [1..n]$ uniformly at random.
    Set $X_{t+1} = k$ with probability $M(k,i)$ and $X_{t+1} = i$ with the complementary probability $1 - M(k,i) = M(i,k)$.
    For~$Y_{t+1}$, distinguish between two cases.
    If~$Y_t=X_t$, set $Y_{t+1}$ = $X_{t+1}$ and otherwise set $Y_{t+1} = k$ with probability $M(k,j)$ and $Y_{t+1} = j$ with the complementary probability $1 - M(k,j) = M(j,k)$.
    The coupling is valid as each challenger is selected with probability~$1/n$ and beats the incumbent with probability $M(k,i)$ or $M(k,j)$, respectively.
    Assume~$X_t\neq~Y_t$.
    Then, conditioned on the challenger $C_t$, the probability of $X_{t+1} = Y_{t+1}$ can be calculated by
    \begin{align*}
        &\PrMath{X_{t+1} = Y_{t+1} \mid X_t = i, Y_t = j, C_t = k}\\
            &=  \begin{cases}
            M(i,j),& \text{for $k = i$}; \\
            M(j,i),& \text{for $k = j$}; \\
            M(k,i)\cdot M(k,j), & \text{otherwise.}
        \end{cases}
    \end{align*}
    Since we select $C_t = k$ uniformly at random, the probability of $X_{t+1} = Y_{t+1}$ can be lower bounded by
    \begin{align*}
        &\PrMath{X_{t+1} = Y_{t+1} \mid X_t = i, Y_t = j} \\
        &= \frac{1}{n} M(i,j) + \frac{1}{n} M(j,i) + \frac{n-2}{n} \sum_{k \in [1..n], k \neq i, k \neq j} M(k,i) \cdot M(k,j) \\
        &= \frac{1}{n} M(i,j) + \frac{1}{n}(1-M(i,j)) + \frac{n-2}{n} \sum_{k \in [1..n], k \neq i, k \neq j} M(k,i) \cdot M(k,j) \\
        &= \frac{1}{n} + \frac{n-2}{n} \sum_{k \in [1..n], k \neq i, k \neq j} M(k,i) \cdot M(k,j)\\
        & \geq \frac{1}{n}
    \end{align*}
    Using Lemma~\ref{lem:prelim:mixing-time}, the two copies of the Markov chain are coupled as soon as $X_t = Y_t$ holds.
    For every $x, y, i, j \in S$ and $t\leq T$ it holds that
    \begin{align*}
        &\PrMath{X_T \neq Y_T \mid X_0 = x, Y_0 = y}\\
        &= \left(1 - \Prob{X_{t+1} = Y_{t+1} \mid X_t = i, Y_t = j}\right)^T\\
        &\leq \left(1 - \frac{1}{n}\right)^T.
    \end{align*}
    
    Let $T = n \ln (1/\epsilon)$.
    Using $\left(1 - \frac{1}{n}\right)^{cn} \leq e^{-c}$ leads to
    \begin{align*}
        \left(1 - \frac{1}{n}\right)^T &= \left(1 - \frac{1}{n}\right)^{n \ln (1/\epsilon)}\\
        &= e^{-\ln(1/\epsilon)}\\
        &\leq \epsilon,
    \end{align*}
    which concludes the proof of Lemma \ref{lem_coupling_res}.
\end{proofE}

Using the mixing time as given in Lemma~\ref{lem_coupling_res}, Proposition~\ref{prop:bin:stationary_dist_bounds} and Corollary~\ref{cor:bin:stationary_dist_bounds}, the next corollary follows directly from Theorem~\ref{thm:prelim:exp-opt-time}.

\begin{corollary}
    Consider the \oneoneEA variant for the Condorcet winner search, $n \in \N$ and $S = [1..n]$.
    Let $i^*_t\in S$ be the underlying Markov chain for $t \in \natnum$ that denotes the current best solution in the $t$-th iteration of the algorithm, $\overline{\pi} = (\pi_1,\ldots,\pi_n)$ its unique stationary distribution and $i^* \in \N$ the Condorcet winner.
    Then, the following expected optimization times for two different constraints of the Condorcet winner search hold.
    \begin{itemize}
        \item[(1)] For all $i \in S$ with $i \neq i^*$, let $M(i^*, i) \geq 1-p_l > \frac{1}{2}$.
        Then, the expected optimization time is at most
        \begin{align*}
            \hspace{9mm}&\tau(\epsilon) \cdot \BigO{\log(1/\pi_{i^*})}/\pi_{i^*}\\ 
            &= n \cdot \ln \left(\frac{1}{\epsilon}\right)\cdot \BigO{\log \left({1+\frac{(n-1)p_l}{1-p_l}}\right)\cdot \left(1+\frac{(n-1)p_l}{1-p_l}\right)}.
        \end{align*}
        \item[(2)] Let 
        \begin{align*}
            p = \frac{o(1)}{n}
        \end{align*}
        and for all $i \in S$ with $i \neq i^*$, let $M(i^*, i) = 1-p$.
        Then, the expected optimization time is at most
        \begin{align*}
            &\tau(\epsilon) \cdot \BigO{\log(1/\pi_{i^*})}/\pi_{i^*}\\ 
            &=(1+o(1)) n \cdot \ln (1/\epsilon)\cdot \BigO{\log (1+o(1))}.
        \end{align*}
    \end{itemize}
    \label{cor_exp_opt_time}
\end{corollary}
\section{The Plackett-Luce Model}
\label{sec:PLBoost}
In Section~\ref{sect:unst-cond-win-stoch}, we assumed an arbitrary statistic model that defines the Matrix $M$ of the Condorcet winner search, but a commonly used way to define the probability that one arm wins against another in Multi-Armed Bandits is to assume an underlying statistic model, e.g. Bradley-Terry or its generalization, the Plackett-Luce Model. This is defined as follows.
\begin{definition}[Plackett-Luce Model] 
Assume we have an underlying utility $u_i \in \mathbb{R}$ for each arm $i \in \firstInts{n}$. Then the probability that arm $i \in \firstInts{n}$ wins in a subset of competitors $S \subset \firstInts{n}$ is defined as
\begin{align*}
    \Prob{i | S} = \frac{u_i}{\sum_{j \in S} u_j}.
\end{align*}
\end{definition}
While such an underlying statistical model clearly defines the probability of one arm being the winner in a single duel, it remains an open question how the winning probabilities will develop over multiple duels, assuming an outcome according to the Plackett-Luce Model for each of the duels. 
We use this ``boosting'' (= multiple duels), to improve the \oneoneEA variant in Algorithm~\ref{alg_1p1EA_Condorcet}.
For the sake of ease, we will first investigate this for the special case of $n=2$ and then generalize the results to a bigger set of arms. While for two arms, any pairwise winning probabilities can be expressed by a Plackett-Luce Model, it guarantees us some structure, like an implicit ordering of the arms, the existence of a Condorcet winner and the absence of any cycles for larger sets of arms.
\subsection{Comparison of two arms}\label{seq:PLBoosting2Arms}
In the following, we will study under which condition the probability that one arm $i \in \firstInts{n}$ wins against another arm $j \in \firstInts{n}, j\neq i$ is boosted by dueling them multiple times $x \in \mathbb{N}$.
\begin{theorem}[Sufficient number of duels]\label{Thm:NecessaryNumberDuelsTwoArms}
    Assume two arms $i,j \in \firstInts{n}$ with the underlying Plackett-Luce Model and $u_i > u_j$. In addition, assume a fixed failure probability $\epsilon \in (0,1)$ that should be a lower bound on the probability that arm $j$ is identified as the winner after $x \in \mathbb{N}$ duels of $i$ and $j$.
    Let the winning probability of arm $i$ fulfill $\frac{u_i}{u_i+u_j}>\frac{x+1}{2x}$.
    Then arm $i$ wins in expectation more often against arm $j$ with probability at least $1-\epsilon$ if the number of duels $x$ between $i$ and $j$ is greater than
    \begin{align*}
        x \geq \frac{2(u_i+u_j)^2}{(u_j-u_i)^2} \ln\left(\frac{1}{\epsilon}\right).
    \end{align*}
\end{theorem}
It is worth noting that we can also write the ratio of the sum and the difference of the utilities as 
\begin{align}\label{eq:fracsumdifference}
    \frac{u_i + u_j}{u_i-u_j} = \frac{1}{p_i - p_j},
\end{align}
where $p_i = \frac{u_i}{u_i+u_j}$ and $p_j = \frac{u_j}{u_i+u_j}$ are the probabilities that arm~$i$, respectively $j$, wins a duel of both arms. Note that equation (\ref{eq:fracsumdifference}) becomes larger if the difference in the winning probabilities of both arms becomes smaller, or in other words, if the quality of the arms is closer together and thus harder to distinguish. The number of necessary duels even grows quadratically in this term, while the allowed failure probability of identifying the best arm correctly only has a logarithmic influence. In particular, if we assume that the winning probability of arm~$i$ is at least $p_i \geq 1-p$, we can conclude for the necessary number of duels to guarantee more wins for arm~$i$ with probability at least $1-\epsilon$
\begin{align}
    x \geq 2\left(\frac{1}{1-2p}\right)^2\ln\left(\frac{1}{\epsilon}\right) \label{eq:lower-bound-x-rho}
\end{align}
if $p$ is chosen such that $1-p > \frac{x+1}{2x}$.
To prove Theorem \ref{Thm:NecessaryNumberDuelsTwoArms}, we will first derive a lower bound on the probability that arm $i$ against arm~$j$ in at least half of the duels.
\begin{lemma}\label{Lem:TwoArmsWinsLowerBound}
    Let $i,j \in \firstInts{n}$ be two arms with an underlying Plackett-Luce Model and $u_i > u_j$ with $\frac{u_i}{u_i+u_j}>\frac{x+1}{2x}$ for $x \in \mathbb{N}$ duels. Then the probability that $i$ wins against $j$ in expectation in $x$ duels is lower bounded by
    \begin{align*}
        \text{Pr}\bigl[ &i \text{ wins against j } \geq \frac{x+1}{2} \text{ times}\bigr] \\
        &\geq 1 - \frac{1}{\exp\left(\frac{x(u_j-u_i)^2}{2(u_i+u_j)^2}\right)}.
    \end{align*}
\end{lemma}
\begin{proof}
By means of the Plackett-Luce Model, we can write the probability that $i$ wins more often against $j$ in $x \in \mathbb{N}$ duels as
\begin{align*}
    &\text{Pr}\big[i \text{ wins against j } \geq \frac{x+1}{2} \text{ times}\big] \\
    &= \sum_{y = \frac{x+1}{2}}^x \binom{x}{y} \left(\frac{u_i}{u_i+u_j}\right)^y \left(\frac{u_j}{u_i+u_j}\right)^{x-y}.
\end{align*}
This is exactly the same as the probability $\Prob{Y \geq \frac{x+1}{2}}$ of a binomial distributed random variable $Y \sim \mathrm{Bin}(x,p)$ with $p=\frac{u_i}{u_i+u_j}$. Note that we can assume w.l.o.g. that $x$ is odd such that we cannot have ties in the number of wins.
By means of Hoeff\-ding's inequality, we get
\begin{align*}
    \Prob{Y \geq \frac{x+1}{2}} 
    %
    %
    &\geq 1 - \Prob{Y - xp \leq \frac{x+1}{2} - xp} \\
    &\overset{\text{Hoeffding}}{\geq} 1 - \exp\left(-\frac{2\left(\frac{x+1}{2} - xp\right)^2}{x}\right) \\
    &\geq 1 - \frac{1}{\exp\left(\frac{2\left(\frac{x}{2} - \frac{xu_i}{u_i+u_j}\right)^2}{x}\right)} \\
    %
    %
    &= 1 - \frac{1}{\exp\left(\frac{x(u_j-u_i)^2}{2(u_i+u_j)^2}\right)}.
\end{align*}
Note that for applying the Hoeffding inequality, we need the condition 
\begin{align*}
    \frac{x+1}{2} &< xp = \frac{xu_i}{u_i+u_j} \\
    \Leftrightarrow ~~ \frac{x+1}{2x} &< \frac{u_i}{u_i+u_j}.
\end{align*}
\end{proof}
Using this result, we can prove the main result about the necessary number of duels to identify arm $i$ as the best one with probability at least $1-\epsilon$ as stated in Theorem \ref{Thm:NecessaryNumberDuelsTwoArms}. The idea of the proof is to set the lower bound on the probability that arm $i$ wins in expectation more than half of the duels equal to $1-\epsilon$ for the desired failure probability $\epsilon$ and to solve the equation for the number of duels $x$. The detailed proof can be found in the appendix.

Furthermore, we can extend our results by comparing the stationary distributions of the induced Markov chain. For this, we first need an auxiliary upper bound on the probability that arm $i$ wins against arm $j$ in at least half of the duels $x$.
\begin{corollary}\label{Cor:TwoArmsWinsUpperBound}
    Let $i,j \in \firstInts{n}$ be two arms with an underlying Plackett-Luce Model with $u_i > u_j$ and $\frac{u_i}{u_i+u_j}>\frac{x+1}{2x}$. Then the probability that $i$ wins against $j$ in expectation in $x \in \mathbb{N}$ duels is upper bounded by the complementary probability
    \begin{align*}
        \text{Pr}\Bigl[ &i \text{ wins against j } \geq \frac{x+1}{2} \text{ times}\Bigr] \\
        &\leq \frac{1}{\exp\left(\frac{x(u_j-u_i)^2}{2(u_j+u_i)^2}\right)}.
    \end{align*}
\end{corollary}
We can now derive the necessary number of duels such that the fraction of the stationary distribution of the induced Markov chain is lower bounded by a $\gamma \in \mathbb{R}$. The detailed proof of the following corollary can be found in the appendix.
\begin{corollary}[Boosting of stationary distribution]\label{cor:BoostingStationaryDis2Arms}
    Let $\bar{\pi}=(\pi_1,\pi_2)$ be the stationary distribution for the Markov chain representing the duels of two arms $1,2$ with an underlying Plackett-Luce Model parameterized by $u_1$ and $u_2$. W.l.o.g. we can assume that $u_1 > u_2$ and in addition, we assume $\frac{u_i}{u_1+u_2}>\frac{x+1}{2x}$. The fraction of the stationary distribution $\frac{\pi_1}{\pi_2}$ is greater than $\gamma \in \mathbb{R}$ if the number of duels $x \in \mathbb{N}$ is greater than
    \begin{align*}
        x > \frac{2(u_2+u_1)^2}{(u_1-u_2)^2} \ln(\gamma +1).
    \end{align*}
\end{corollary}
Note that, in addition to the term we already discussed in equation \ref{eq:fracsumdifference}, the necessary budget grows logarithmically with $\gamma$. A larger~$\gamma$ means that in the stationary distribution $\pi_1$ and $\pi_2$ are further apart from each other, thus a larger number of necessary duels intuitively makes sense here to further boost the probabilities that arm $1$ is identified as a winner after the $x$ duels. Analogously to the above, we can derive for a winning probability of arm $1$ of at least $p_1 \geq 1-p$ a necessary number of duels 
\begin{align*}
    x > 2\left(\frac{1}{1-2p}\right)^2\ln(\gamma +1)
\end{align*}
if $1-p > \frac{x+1}{2x}$ holds.
In particular, we have reached a boosting of the fraction if we choose $\gamma \geq \frac{u_1}{u_2}$, which leads to the necessary number of duels of 
\begin{align*}
    x > \frac{2(u_1+u_2)^2}{(u_2-u_1)^2}\ln\left(\frac{u_1+u_2}{u_2}\right).
\end{align*}
For a better understanding, we will regard the special case of $3$ duels of two arms in the following and will study the development of the winning probability for the better arm when we execute multiple instead of only one duel.
\begin{example}[Best of 3 duels]\label{ex:BestOf3Duels}
    Assume two arms $1,2$ with an underlying Plackett-Luce Model and $u_1>u_2$. The probability that arm $1$ wins more often than arm $2$ in three duels is given by 
    \begin{align*}
        &\Prob{1\text{ wins against }2 \geq 2 \text{ times}} = \frac{3u_1^2u_2 + u_1^3}{u_1^3 + 3u_1^2u_2 + 3u_1u_2^2 + u_2^3}.
    \end{align*}
    This leads to the following fraction of the unique stationary distribution
    \begin{align*}
        \frac{\pi_1}{\pi_2} &= \frac{3u_1^2u_2 + u_1^3}{3u_1u_2^2 + u_2^3}, 
    \end{align*}
    which boosts the probability that arm $1$ wins in expectation against arm $2$ in comparison to only one duel. For more details, see the appendix.
\end{example}
Using Corollary~\ref{cor:bin:stationary_dist_bounds}, Theorem~\ref{Thm:NecessaryNumberDuelsTwoArms}, and (\ref{eq:lower-bound-x-rho}), it is possible to retrieve the following corollary for the stationary distribution of the Condorcet winner.
For the \oneoneEA variant in Algorithm~\ref{alg_1p1EA_Condorcet}, it shows that replacing the singular query to determine the winner in each round with the ``boosted'' variant of the Plackett-Luce Model enables us to get a stationary distribution $\pi_{i^*}$ of the Condorcet winner $i^*$ close to $1$, even for less strict constraints on the probabilities~$M(i^*,i)$.
The tradeoff is that, in each iteration, multiple duels must be played between the \emph{incumbent} and the \emph{challenger}.
\begin{corollary}\label{cor:plackett_luce}
    Assume that the underlying Matrix $M$ of the Condorcet winner search is defined by the Plackett-Luce Model.
    Consider the \oneoneEA variant for the Condorcet winner search, $n \in \N$, and $S = [1..n]$ be the state space.
    Replace the Query in Line~5 of Algorithm~\ref{alg_1p1EA_Condorcet} such that the arms play $x$ duels against each other and the arm that wins at least $\frac{x+1}{2}$ duels, is assigned as the winner.
    Let $M_t$ be the underlying Markov chain for $t \in \natnum$ with $M_t = i_{t}^*$ and $i_{t}^* \in S$ as the current best solution in the $t$-th iteration of the algorithm, $\overline{\pi} = (\pi_1,\ldots,\pi_n)$ its unique stationary distribution and $i^* \in \N$ the unique Condorcet winner.
    Suppose for all $i$ for all $i \in S$ with $i \neq i^*$, that $M(i^*, i) \geq 1/2 + \delta$ with $\delta < 1/2$ holds.
    Then the stationary distribution of the Condorcet winner can be calculated for different values of $x$ as follows.
    \begin{itemize}
        \item[1.] Let 
        \begin{align*}
            x_1 = \frac{1}{2}\left(\frac{1}{\delta}\right)^2\ln\left(\frac{n}{o(1)}\right).
        \end{align*}
        Then, the stationary distribution of the Condorcet winner is
        \begin{align*}
            \pi_{i^*,1} = 1-o(1).
        \end{align*}
        \item[2.] Let 
        \begin{align*}
            x_2 = \left(\frac{1}{\delta}\right)^2\ln\left(n\right).
        \end{align*}
        Then, the stationary distribution of the Condorcet winner is
        \begin{align*}
            \pi_{i^*,2} = 1-\Theta(1/n).
        \end{align*}
    \end{itemize}
    \label{cor:bin:stationary_dist_bounds:boosting}
\end{corollary}
\subsection{Comparison of $n$ arms}\label{seq:PLBoostingNArms}
Let us now consider the case of a duel of more than $2$ arms. To be precise, we assume a fixed, but random set $S \subseteq \firstInts{n}$ of arms that will be pulled in $x \in \mathbb{N}$ duels.
\begin{theorem}\label{thm:comparisonNarms}
    Let $S:=\{1,\dots,|S|\}$ be a set of arms with underlying Plackett-Luce model parameterized by $\{u_1, \dots, u_{|S|}\}$. Then we can bound the probability that arm $i \in S$ is in expectation the best arm after $x\in \mathbb{N}$ duels by
    \begin{align*}
        \mathrm{Pr}[i &\text{ best arm in $x$ duels}]\\
        &~~\leq \sum_{d = \left\lfloor\frac{x}{|S|}\right\rfloor + 1}^{\left\lfloor\frac{x}{2}\right\rfloor} \Bigg[\binom{x}{d} \left(\frac{u_i}{\sum_{l \in S}u_l}\right)^d \left(\frac{\sum_{l \in S}u_l - u_i}{\sum_{l \in S}u_l}\right)^{x-d} \\
        &~~\prod_{j \in S\backslash\{i\}}\exp\left(-(x-d)\text{KL}\left(\frac{d-1}{x-d} || \frac{u_i}{\sum_{l \in S}u_l} \right)\right) \\
        &~~\cdot \left(\frac{\sum_{l \in S}u_l-u_i-u_j}{\sum_{l \in S}u_l- u_j}\right)^{x-2d+1} \Bigg]\\
        &~~+ \exp\left(-x \text{KL}\left(\frac{x-\lfloor x/2\rfloor -1}{x} || \frac{\sum_{l \in S}u_l - u_i}{\sum_{l \in S}u_l}\right)\right).
    \end{align*}
    and
    \begin{align*}
        \mathrm{Pr}[i &\text{ best arm in $x$ duels}]\\
        &~~\geq \sum_{d = \left\lfloor\frac{x}{|S|}\right\rfloor + 1}^{\left\lfloor\frac{x}{2}\right\rfloor} \Bigg[\binom{x}{d} \left(\frac{u_i}{\sum_{l \in S}u_l}\right)^d \left(\frac{\sum_{l \in S}u_l - u_i}{\sum_{l \in S}u_l}\right)^{x-d} \\
        &~~\prod_{j \in S\backslash\{i\}}\frac{1}{\sqrt{8(d-1)(1-\frac{d-1}{x-d})}}\exp\left(-(x-d)\text{KL}\left(\frac{d-1}{x-d} || \frac{u_i}{\sum_{l \in S}u_l} \right)\right) \\
        &~~\cdot \left(\frac{\sum_{l \in S}u_l-u_i-u_j}{\sum_{l \in S}u_l- u_j}\right)^{x-d} \Bigg] + \frac{1}{\sqrt{8(x-\lfloor\frac{x}{2}\rfloor-1)(1-\frac{x-\lfloor\frac{x}{2}\rfloor-1}{x}}}\\
        &~~\cdot \exp\left(-x \text{KL}\left(\frac{x-\lfloor x/2\rfloor -1}{x} || \frac{\sum_{l \in S}u_l - u_i}{\sum_{l \in S}u_l}\right)\right).
    \end{align*}
\end{theorem}

Note that for the case that arm $i$ wins more often than any other arm a different number of wins can be sufficient, depending on the quality and the number of wins of the other competitors. If the probability to win a duel is equally distributed over the competitors, arm $i$ can already be the one with the most wins if it wins $\left\lfloor \frac{x}{|S|}\right\rfloor +1$ times. On the other hand, if the second-best arm wins all other duels except the ones arm $i$ wins, arm $i$ needs $\left\lfloor \frac{x}{2} \right\rfloor +1$ wins to get identified as the best arm. Reasoning by that, we get a sum of the possible wins of arm $i$ in the lower and upper bounds for the winning probabilities, and no closed form. The detailed proof can be found in the appendix. One can see the development of these bounds for the probability that arm $i$ wins most of the time for different underlying Plackett-Luce utilities over multiple duels $x$ in Figure \ref{fig:ProbabilityBoundsBestArmNArms}.
\begin{figure}[ht!]
    \includegraphics[width=\columnwidth]{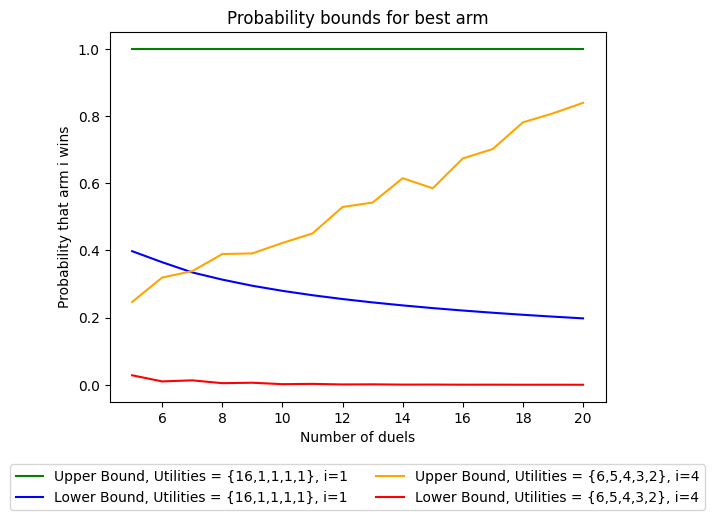}
    \caption{Plot of the upper and lower bounds on the probability that arm $i$ is the best in expectation for different arms and instantiations of the Plackett-Luce model.}
    \label{fig:ProbabilityBoundsBestArmNArms}
    \Description[Upper bound of probability that arm $i$ is the best in expectation converges to $1$ for larger number of duels.]{Upper bound of probability that arm $i$ is the best in expectation converges to $1$ for larger number of duels. Analogously the lower bound converges to $0$. This occurs faster for larger differences of the utilities.}
\end{figure}
\section{Stochastic Queries: Ant Colony Optimization}
\label{sec:aco_analysis}

\newcommand{\taumin}{\tau_{\mathrm{min}}}

In previous chapters, we studied the distribution of evolutionary algorithms over the set of arms. A class of algorithms which provides such a distribution explicitly are the estimation of distribution algorithms (EDAs).

We consider the following algorithm, an adaptation of a specific EDA called MMAS-ib~\cite{DBLP:series/ncs/KrejcaW20} (Max-Min Ant-System with iteration best updated) to the setting of Condorcet winner search. We maintain a vector $\tau \in [0,1]^n$ of \emph{pheromones} or \emph{marginal probabilities}. These are getting updated according to an update scheme depending on a \emph{learning rate} $\rho$.

For all $\tau \in [0,1]^n$ with $\sum_{i=1}^n \tau_i =1$, we let $R(\tau)$ denote the random distribution on $\firstInts{n}$ such that, for all $i \in \firstInts{n}$, it holds that $\Prob{R(\tau) = i} =\tau_i$. Intuitively, each arm $i$ has a pheromone value $\tau_i$ associated with it, representing the probability of this arm being chosen. We use this distribution to sample, in each iteration, two arms to compare.

Initially, all arms $i$ have the same marginal probability $\tau_i = 1/n$. After each iteration, all marginal probabilities are diminished by a factor of $(1-\rho)$ (known as \emph{pheromone evaporation}); only the winner of this iteration gets an increase of its marginal probability by $\rho$. In order to prevent premature convergence, we enforce a lower bound of $\taumin$ for all pheromone values.\footnote{We do not explicitly enforce an upper bound, since the lower bound is sufficient to ensure exploration of arms.} Note that this is crucial, since the starting probabilities are so low that there is a non-zero chance to never sample an arm otherwise. We define the pheromone update function formally as follows.
$$
\mathrm{update}(\tau,i^*)_i = \begin{cases}
    \max(\tau_i(1-\rho),\taumin), &\mathrm{if }\;i \neq i^*;\\
    \tau_i(1-\rho) + \rho, &\mathrm{otherwise.}    
\end{cases}
$$
Furthermore, we use a function normalize$(\tau)$ to normalize a vector of pheromone values to sum up to $1$ by dividing each entry by the total sum of the vector. Note that the update scheme ensures that, as long as the bound $\taumin$ is not reached, no normalization is necessary. Algorithm~\ref{alg_ACO_Condorcet} shows the pseudo code for MMAS-ib.

\begin{algorithm}[htp]
    \textbf{Parameters:} $n \in \natnum$, $\rho \in (0,1/2)$, $\taumin \in (0,1/2)$\;
    For $i \in \firstInts{n}$: $\tau_i^0$ $\assign$ $1/n$\;
    \For{$t=0$ \KwTo $\infty$}{
        $i$ $\sim$ $R(\tau^t)$\;
        $j$ $\sim$ $R(\tau^t)$\;
        $i^*$ $\assign$ Query$(i,j)$\;
        $\tau^{t+1}$ $\assign$ normalize$($update$(\tau^{t+1},i^*))$\;
    }
    \caption{MMAS-ib, Condorcet Winner Search}
    \label{alg_ACO_Condorcet}
\end{algorithm}

We start with a lemma quantifying the impact of normalization.

\begin{lemma}\label{lem:updateRuleLowerBound}
    Let $n \in \natnum$ and $\taumin \in (0,1/2)$ be given. For any pheromone vector $\tau \in [\taumin/(1+\rho),1]^n$ with sum $1$ and any $i^* \in \firstInts{n}$ we have that the sum of update$(\tau^{t+1},i^*))$ is at most $1+2\taumin\rho n$. In particular, every component of 
    $$
    \mathrm{normalize}(\mathrm{update}(\tau^{t+1},i^*))
    $$
    is lower bounded by $\taumin/(1+2\taumin \rho n)$.
\end{lemma}
\begin{proof}
    Before the update, each component is lower bounded by $\taumin/(1+\rho)$, so after pheromone evaporation, each is lower bounded by $\taumin(1-\rho)/(1+\rho)$. Thus, the margin enforced by the pheromone update rule increases each component by up to
    $$
    \taumin - \taumin(1-\rho)/(1+\rho) = \taumin\frac{1+\rho-(1-\rho)}{1+\rho} = \frac{2\taumin\rho}{1+\rho}.
    $$
    Thus, summing over all $n$ components and neglecting the division by $1+\rho$, the sum is increased by the update rule by at most $2\taumin \rho n$.
\end{proof}

In passing, we note that the MMAS-ib will, with positive probability, encounter the lower bound with the pheromone value of any fixed arm, regardless of this arm's competitiveness and regardless of how low the pheromone bound $\taumin$ is.

\begin{proposition}
    Let $n \in \natnum$ and $\rho,\taumin \in (0,1/2)$ be given. Consider the MMAS-ib variant for the Condorcet winner search given in Algorithm~\ref{alg_ACO_Condorcet} and let $i \in \firstInts{n}$ be given. Then the probability that $i$ is not chosen in any iteration before the pheromone associated with $i$ attains the lower bound $\taumin$ is at least $\exp(-2/(n\rho))$.
\end{proposition}
\begin{proof}
    We ignore the lower bound of $\taumin$ and compute the probability of never choosing $i$. After $t$ iterations, the pheromone value will be $(1/n)(1-\rho)^t$. Thus, using $\forall x \in (0,1/2):\exp(-2x) \leq 1-x$, the probability of not choosing $i$ in any iteration is
    \begin{align*}
    \prod_{t=0}^\infty \left(1-\frac{(1-\rho)^t}{n}\right) &\geq \prod_{t=0}^\infty \exp\left(-2\frac{(1-\rho)^t}{n}\right) \\
    & = \exp\left(-2/n \; \sum_{t=0}^\infty (1-\rho)^t\right)\\
    & = \exp\left(-2/(n\rho) \right).
    \end{align*}
    Since with this at least probability the pheromone value decreases to $0$ in the limit, this probability is a lower bound of the algorithm ever attaining $\taumin$ with the pheromone value for $i$.
    %
\end{proof}

As the next theorem shows, MMAS-ib identifies a Condorcet winner already from a small bias. Note that an upper bound for any pheromone value is $1-(n-1)\taumin/(1+\rho)$, since all other pheromone values observe a lower bound.

\begin{theorem}
    Let $n \in \natnum$, $\taumin,\rho \in (0,1/n)$ with $\taumin \leq 1/(10n)$.
    Consider the MMAS-ib variant for the Condorcet winner search given in Algorithm~\ref{alg_ACO_Condorcet}. Let $S = \firstInts{n}$ be the set of options with $i^* \in S$ the Condorcet winner. Suppose there is $p \in [0,1/4]$ such that 
    $$
    \forall i \in S \setminus \{i^*\}: M(i^*,i) \geq 1-p.
    $$
    Then it takes the algorithm an expected time of 
    $$
    \BigO{\frac{1}{\taumin\rho} + \frac{\log(1/p)}{\rho}}
    $$
    until reaching a marginal probability of at least $1-3p-n\taumin$ for the Condorcet winner for the first time.
\end{theorem}
\begin{proof}
Let $s = 2\taumin n$. Now Lemma~\ref{lem:updateRuleLowerBound} gives that all pheromone values are always lower bounded by $\taumin/(1+s\rho)$. Furthermore, we know that the algorithm normalizes with a factor at most $1+s\rho$.

We let, for all $t \in \natnum$, $X_t = 1-2p-s-\tau_{i^*}^t$. 
Let now $t$ be given and denote $\tau = \tau_{i^*}^t$ the ``current'' pheromone value on the Condorcet winner at the beginning of iteration $t$. Let $A$ be the event that, in iteration $t$, the arm $i^*$ is chosen and wins the comparison. We make the distinction whether arm $i^*$ is chosen as the first arm. If so, then it wins with probability at least $1-p$; if not, then there's a probability of $\tau$ to choose arm $i^*$ as the second arm. Overall, we have
$$
\Prob{A} \geq  (\tau + (1-\tau)\tau)(1-p).
$$

Now we get
\begin{align*}
    &\Ex{X_{t} - X_{t+1} \mid X_t}\\
    &= \tau_{i^*}^{t+1} - \tau\\
    &\geq (\tau(1-\rho) + \Prob{A} \rho)/(1+s\rho) - \tau\\
    &\geq (\tau(1-\rho) + (\tau + (1-\tau)\tau)(1-p) \rho - \tau - \tau s\rho)/(1+s\rho)\\
    &= (-\tau\rho-s\tau\rho + \tau\rho(1+1-\tau)(1-p))/(1+s\rho)\\ 
    &= \tau\rho((2-\tau)(1-p) - 1-s)/(1+s\rho)\\ 
    &= \tau\rho(1-s-2p - \tau + p\tau)/(1+s\rho)\\ 
    &= \tau\rho(X_t + p\tau)/(1+s\rho)\\ 
    &\geq \tau\rho X_t/2.
\end{align*}
For $\tau \in (\taumin/2,1/6)$, we can further lower bound $X_t$ by $3/10$ and thus obtain a drift of at least
$
\taumin \rho\cdot(3/40).
$
Note that the process cannot make steps larger than $\rho$. Thus, using the \emph{additive drift theorem} with overshooting~\cite{Krejca2019}, we see that the time until the process $X_t$ reaches a value of at least $1/6$ is in $\BigO{1/(\taumin\rho)}$.

Furthermore, for $\tau \in [1/10,1/6]$, the process $(X_t)_{t \in \natnum}$ exhibits negative drift towards $1$, of order $\Theta(\rho)$; recall that the process cannot make steps of size larger than $\rho$ by definition. Thus, the \emph{negative drift theorem}~\cite{Oliveto2011} gives that the process will not fall below~$1/10$ in a time exponential in $1/\rho$.

Finally, while $\tau \geq 1/10$, the process $(X_t)_{t \in \natnum}$ exhibits \emph{multiplicative drift} towards $0$ with factor $\rho/20$. Thus, we can use the multiplicative drift theorem \cite{10.1145/1830483.1830748} to see that first $t$ such that $X_t \leq p$ is, in expectation, at most $\BigO{\log(1/p)/\rho}$.

Taking these bounds together with a simple restart argument shows the theorem.
\end{proof}

Note that the proof also shows that the pheromone of the Condorcet winner will generally stay high, as there is drift towards these high values.
\section{Conclusion}

In this paper, we were interested in the performance of evolutionary algorithms in a setting of Multi-Armed Bandits, in order to shed light on the way such algorithms can explore and exploit knowledge obtained about the search space. As we saw, a simple \oneoneEA struggles with finding stochastic differences between options, even when one option is clearly better. In contrast, the cumulative nature of a simple EDA leads to a much better focus on a better option.

For heuristic search, the structure of the space to be searched is important, typically neighborhoods and local operators are defined. In contrast, our model for Multi-Armed Bandits assumes no relation between options. We consider this a first study for search heuristics in this setting and believe the settings of Combinatorial Bandits \cite{Mohajer2017, Wenbo2019, Wenbo2020, Chen2020, Saha2020, Haddenhorst2021} to be interesting for future work, incorporating the structure of the search space into the model.


\bibliographystyle{ACM-Reference-Format}
\bibliography{_ref.bib}

\section*{Acknowledgements}
This research was (partially) funded by the HPI Research School on Foundations of AI (FAI).

This research was supported by the Ministry of Culture and Science (NRW, Germany) as part of the Lamarr Fellow Network and by the Federal Ministry of Research, Technology and Space (BMFTR) under grant no. 16IS24057A. We also gratefully acknowledge funding from the European Research Council (ERC) under the ERC Synergy Grant Water-Futures (Grant agreement No. 951424). This publication reflects the views of the authors only.

\clearpage
\newpage

\appendix
\onecolumn
\section{Stochastic Queries}
\label{sect:app:condorcet-win-search}
\subsection{Insights on the Mixing Time and Coupling}
In the following, we define, according to \cite{Mitzenmacher2017ProbComp}, the mixing time of Markov chains and how to bound them using coupling.

\begin{definition}
    Let $\overline{\pi}$ be the stationary distribution of an ergodic Markov chain with state space S.
    Let $p_x^t$ represent the distribution of the state of the chain starting at state $x$ after $t$ steps.
    We define
    \begin{align*}
        \Delta_x(t) = || p_x^t - \overline{\pi}||; \qquad \Delta(t) = \max_{x \in S} \Delta_x(t)
    \end{align*}
    with the variation distance $||.||$.
    That is, $\Delta_x(t)$ is the variation distance between the stationary distribution and $p_x^t$ and $\Delta(t)$ is the maximum of these values over all states $x$.
    Let $\epsilon > 0$.
    We define
    \begin{align*}
        \tau_x(\epsilon) = \min \{t: \Delta_x(t) \leq \epsilon\}; \qquad \tau(\epsilon) = \max_{x \in S} \tau_x(\epsilon).
    \end{align*}
    That is, $\tau_x(\epsilon)$ is the first step $t$ at which the variation distance between $p_x^t$ and the stationary distribution is less than $\epsilon$, and $\tau(\epsilon)$ is the maximum of these values over all states $x$.
\end{definition}
As a function over $\epsilon$, $\tau(\epsilon)$ is called the \emph{mixing time} of the Markov chain.

Our next goal is to determine how the mixing time converges for small $\epsilon$, which results in the number of iterations the $(1+1)$ EA has to run to result in a distribution close to its unique stationary one.
Following \cite[Definition~11.3]{Mitzenmacher2017ProbComp}, we use the coupling technique of Markov chains to achieve that.

\begin{definition}
    A coupling of a Markov chain $M_t$ with state space~$S$ is a Markov chain $Z_t = (X_t, Y_t)$ on the state space $S \times S$ such that:
    \begin{align*}
        \PrMath{X_{t+1} = x' \mid Z_t = (x,y)} = \PrMath{M_{t+1} = x' \mid M_t = x}\\
        \PrMath{Y_{t+1} = y' \mid Z_t = (x,y)} = \PrMath{M_{t+1} = y' \mid M_t = y}.
    \end{align*}
\end{definition}

Using the defined coupling of a Markov Chain, it is possible to bound its mixing time with the following lemma in \cite[Lemma 11.2]{Mitzenmacher2017ProbComp}.

\begin{lemma}
    Let $Z_t = (X_t, Y_t)$ be a coupling of a Markov Chain on a state space $S$. Suppose that there exists a $T$ such that, for every $x, y \in S$
    \begin{align*}
        \PrMath{X_T \neq Y_T \mid X_0 = x, Y_0 = y} \leq \epsilon.
    \end{align*}
    Then $\tau(\epsilon) \leq T$.
    That is, for any initial state, the variation distance between the distribution of the state of the chain after T steps and the stationary distribution is at most $\epsilon$.
\end{lemma}
\section{Boosting of Plackett-Luce}
\subsection{Preliminaries}
\begin{theorem}[Hoeffdings inequality]
    Let $Z_1, \dots, Z_n$ be independent random variables on $\mathbb{R}$ such that $a_i \leq Z_i \leq b_i$ with probability one. If $S_n = \sum_{i=1}^n Z_i$ then for all $t > 0$

    \begin{align*}
        \Prob{ |S_n - \mathbb{E}[S_n] | \geq t} \leq \exp\left(-2t^2 / \sum(b_i-a_i)^2\right).
    \end{align*}
\end{theorem}
\subsection{Proofs of Section \ref{seq:PLBoosting2Arms}}\label{sec:ProofsSection6.1}
\begin{proof}[Proof of Theorem \ref{Thm:NecessaryNumberDuelsTwoArms}]
    If we set the right-hand sight of Lemma \ref{Lem:TwoArmsWinsLowerBound} equal to a desired success probability $1- \epsilon$, we can solve it for $x$ and derive the sufficient number of duels to achieve the desired success probability. 
    %
    \begin{align*}
        &1 - \frac{1}{\exp\left(\frac{x(u_j-u_i)^2}{2(u_i+u_j)^2}\right)} \geq 1- \epsilon \\
        \Leftrightarrow ~~ &\frac{1}{\exp\left(\frac{x(u_j-u_i)^2}{2(u_i+u_j)^2}\right)} \leq \epsilon \\
        \Leftrightarrow ~~ &x \geq \frac{2(u_i+u_j)^2}{(u_j-u_i)^2} \ln\left(\frac{1}{\epsilon}\right)
    \end{align*}
\end{proof}
\begin{proof}[Proof of Corollary \ref{cor:BoostingStationaryDis2Arms}]
    According to Theorem 7.10 in \cite{Mitzenmacher2017ProbComp}, the unique stationary distribution $\bar{\pi}$ is characterized by 
    \begin{align*}
        \pi_1 + \pi_2 &= 1 \\
        \frac{\pi_1}{\pi_2} &= \frac{P_{2,1}}{P_{1,2}},
    \end{align*}
    where $P_{i,j} = \frac{1}{2}\frac{u_j}{u_1+u_2}$ are the transition probabilities of the induced Markov Chain from the given Plackett-Luce Model. By using Lemma \ref{Lem:TwoArmsWinsLowerBound} and Corollary \ref{Cor:TwoArmsWinsUpperBound}, we can bound this fraction by
    \begin{align*}
        \frac{\pi_1}{\pi_2} &\geq \frac{\exp\left(\frac{x(u_2-u_1)^2}{2(u_1+u_2)^2}\right) - 1}{\exp\left(\frac{x(u_2-u_1)^2}{2(u_1+u_2)^2}\right)}\cdot\exp\left(\frac{x(u_1-u_2)^2}{2(u_1+u_2)^2}\right) \\
        &= \exp\left(\frac{x(u_2-u_1)^2}{2(u_1+u_2)^2}\right) -1,
    \end{align*}
    since $(u_1-u_2)^2 = (u_2-u_1)^2$. We want this lower bound of the fraction of the stationary distribution to be less than $\gamma$ and get by simple transformations
    \begin{align*}
        \exp\left(\frac{x(u_2-u_1)^2}{2(u_1+u_2)^2}\right) -1 &> \gamma \\
        \Leftrightarrow ~~~\frac{x(u_2-u_1)^2}{2(u_1+u_2)^2} &> \ln(\gamma +1) \\
        \Leftrightarrow ~~~ x &> \frac{2(u_1+u_2)^2}{(u_2-u_1)^2}\ln(\gamma+1).
    \end{align*}
\end{proof}
\begin{proof}[Proof of Example \ref{ex:BestOf3Duels}]
    If we consider the arm duels as a Markov Chain with transition probabilities induced by the Plackett-Luce Model, the fraction of the stationary distribution matches the fraction of the utilities
    \begin{align*}
        \frac{\pi_1}{\pi_2} = \frac{u_1}{2(u_1+u_2)}\cdot\frac{2(u_1+u_2)}{u_2} = \frac{u_1}{u_2}
    \end{align*}
    If we play a duel multiple times we have the winning probabilities
    \begin{align*}
        &\Prob{1\text{ wins against }2 \geq 2 \text{ times}} \\
        &=3\left(\left(\frac{u_1}{u_1+u_2}\right)^2\left(\frac{u_2}{u_1+u_2}\right)\right) + \left(\frac{u_1}{u_1+u_2}\right)^3 \\
        &= \frac{3u_1^2u_2 + u_1^3}{(u_1+u_2)^3}\\
        &= \frac{3u_1^2u_2 + u_1^3}{u_1^3 + 3u_1^2u_2 + 3u_1u_2^2 + u_2^3}.
    \end{align*}
    Thus, we obtain as stationary distribution
    \begin{align*}
        \frac{\pi_1}{\pi_2} &= \frac{3u_1^2u_2 + u_1^3}{3u_1u_2^2 + u_2^3}.
    \end{align*}
    We can derive that we boost the probabilities if
    \begin{align*}
        &\frac{u_1}{u_2} < \frac{3u_1^2u_2 + u_1^3}{3u_1u_2^2 + u_2^3} \\
        \Leftrightarrow ~~~& 3u_1^2u_2^2 + u_1u_2^3 < 3u_1^2u_2^2 + u_1^3u_2 \\
        \Leftrightarrow ~~~& u_2^2 < u_1^2 \\
        \Leftrightarrow ~~~& u_2 < u_1.
    \end{align*}
\end{proof}
\subsection{Plots of results of Section \ref{seq:PLBoosting2Arms}}
To get an intuition about the theoretically derived probabilities in section \ref{sec:PLBoost}, we show some plots of example Plackett-Luce Models and the corresponding winning probabilities of each arm for a specific number of duels in the following.
%
\begin{figure}[H]
    \centering
    \label{fig:winprobs2arms}
    \includegraphics[width=0.6\columnwidth]{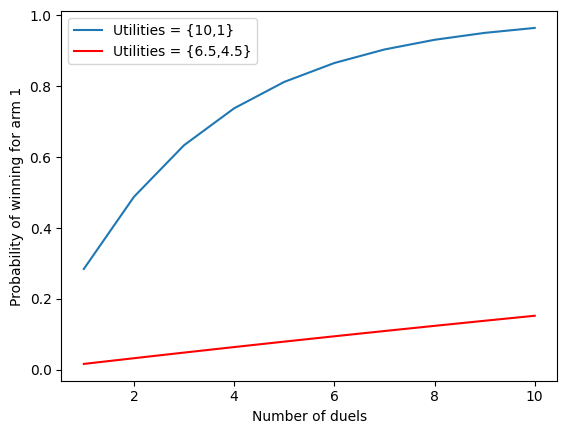}
    \caption{Lower bound for the probability that arm $1$ is in expectation the winner over arm $2$ after up to $10$ duels for different instantiations of the underlying Plackett-Luce Model.}
    \Description[Lower bound for the probability that arm $1$ is in expectation the winner over arm $2$ after up to $10$ duels increases for larger number of duels.]{Lower bound for the probability that arm $1$ is in expectation the winner over arm $2$ after up to $10$ duels increases for larger number of duels. Higher bound for larger difference in the utilities.}
\end{figure}
\begin{figure}[H]
    \centering
    \label{fig:atmostwinprobs}
    \includegraphics[width=0.6\columnwidth]{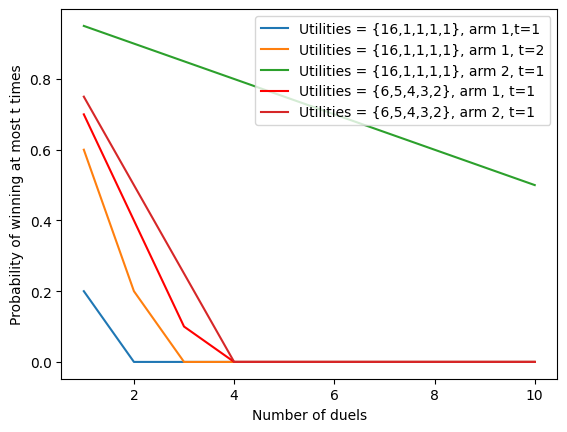}
    \caption{Lower bound for the probability that a specific arm wins in expectation at most $t$ times over all other arms at after up to $10$ duels for different instantiations of the underlying Plackett-Luce Model.}
    \Description[Lower bound for the probability that a specific arm wins in expectation at most $t$ times over all other arms at after up to $10$ duels decreases for larger number of duels.]{Lower bound for the probability that a specific arm wins in expectation at most $t$ times over all other arms at after up to $10$ duels decreases for larger number of duels.}
\end{figure}
\begin{figure}[H]
    \centering
    \label{fig:atleastwinprobs}
    \includegraphics[width=0.6\columnwidth]{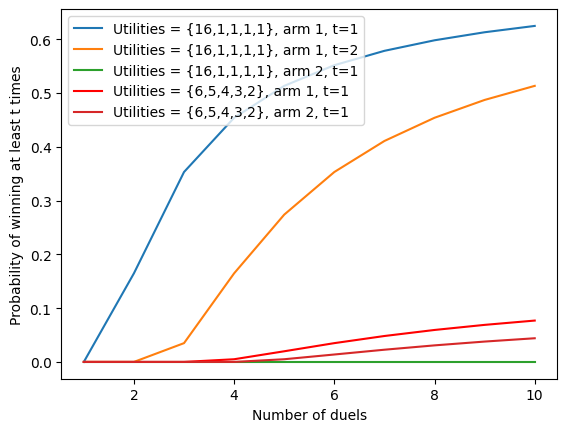}
    \caption{Lower bound for the probability that a specific arm wins in expectation at least $t$ times over all other arms at after up to $10$ duels for different instantiations of the underlying Plackett-Luce Model.}
    \Description[Lower bound for the probability that a specific arm wins in expectation at least $t$ times over all other arms at after up to $10$ duels increases for larger number of duels.]{Lower bound for the probability that a specific arm wins in expectation at least $t$ times over all other arms at after up to $10$ duels increases for larger number of duels. Higher bound for larger difference in the utilities.}
\end{figure}
%

\subsection{Proof of Section \ref{seq:PLBoostingNArms}}\label{sec:ProofsSection6.2}
\begin{proof}[Proof of Theorem \ref{thm:comparisonNarms}] Probability that arm $i$ in set $S$ is identified as the best arm in $x \in \mathbb{N}$ duels is given by
\begin{align*}
    \text{Pr}[i &\text{ best arm in $x$ duels}] \\
    =& \sum_{d = \left\lfloor\frac{x}{|S|}\right\rfloor + 1}^{\left\lfloor\frac{x}{2}\right\rfloor} \Big[ \Prob{ i \text{ wins } = d \text{ times}} \\
    &~~\cdot \prod_{j \in S\backslash \{i\}} \Prob{ j \text{ wins } \leq d - 1 \text{ times} ~|~ i \text{ wins } d \text{ times}}\Big] \\
    &~~+ \Prob{ i \text{ wins } \geq \left\lfloor\frac{x}{2}\right\rfloor +1 \text{ times}}\\
    =&\sum_{d = \left\lfloor\frac{x}{|S|}\right\rfloor + 1}^{\left\lfloor\frac{x}{2}\right\rfloor} \Big[ \Prob{ i \text{ wins } = d \text{ times}} \\
    &~~\cdot \prod_{j \in S\backslash \{i\}} \frac{\Prob{ j \text{ wins } \leq d - 1 \text{ times} ~\cap~ i \text{ wins } d \text{ times}}}{\Prob{ i\text{ wins } d \text{ times }}}\Big] \\
    &~~+ \Prob{ i \text{ wins } \geq \left\lfloor\frac{x}{2}\right\rfloor +1 \text{ times}}.
\end{align*}
The probability that arm $i$ wins \textit{exactly} $d$ times is fixed as
\begin{align*}
        \text{Pr}[ i \text{ wins in S } &= d \text{ of } x \text{ times}]\\
        &= \binom{x}{d}\left(\frac{u_i}{\sum_{j \in S} u_j}\right)^d \left(\frac{\sum_{j \in S}u_j - u_i}{\sum_{j \in S} u_j}\right)^{x-d}.
    \end{align*}
In addition, we can use the distribution function of the multinomial distribution to derive
\begin{align*}
    &\frac{\Prob{ j \text{ wins } \leq d - 1 \text{ times} ~\cap~ i \text{ wins } d \text{ times}}}{\Prob{ i\text{ wins } d \text{ times }}} \\
    &~~= \frac{\sum_{k=0}^{d-1}\binom{x}{d}\left(\frac{u_i}{\sum_{l \in S}u_l}\right)^d \binom{x-d}{k}\left(\frac{u_j}{\sum_{l \in S}u_l}\right)^k \left(\frac{\sum_{l \in S}u_l - u_i - u_j}{\sum_{l \in S}u_l} \right)^{x-d-k}}{\binom{x}{d}\left(\frac{u_i}{\sum_{l \in S}u_l}\right)^d \left(\frac{\sum_{l \in S}u_l - u_i}{\sum_{l \in S}u_l}\right)^{x-d}} \\
    &~~= \frac{\sum_{k=0}^{d-1}\binom{x-d}{k}\left(\frac{u_j}{\sum_{l \in S}u_k}\right)^k \left(\frac{\sum_{l \in S}u_l-u_j}{\sum_{l \in S}u_l}\right)^{x-d-k}\overbrace{\left(\frac{\sum_{l \in S}u_l-u_i-u_j}{\sum_{l \in S}u_l- u_j}\right)^{x-d-k}}^{(*)}}{\left(\frac{\sum_{l \in S}u_l - u_i}{\sum_{l \in S}u_l}\right)^{x-d}}.
\end{align*}
The next step is to lower and upper bound all of the terms contained in the above equations. First of all, we can estimate
\begin{align*}
    (*) &\leq \left(\frac{\sum_{l \in S}u_l-u_i-u_j}{\sum_{l \in S}u_l- u_j}\right)^{x-2d+1} \text{ and}\\
    (*) &\geq \left(\frac{\sum_{l \in S}u_l-u_i-u_j}{\sum_{l \in S}u_l- u_j}\right)^{x-d}.
\end{align*}
Using as tail bounds for the binomial distribution the equations (2) and (5) from \cite{TailBounds} and the fact that for the cumulative distribution function $F(d,x,p) := \Prob{ Y \sim Bin(x,p) \geq d}$, we have $\Prob{Y \leq d} = F(x-d, x, 1-p)$, we get the following upper and lower bounds.
\begin{align*}
    \Prob{ Y \leq d} &\leq \exp\left( -x KL\left(\frac{d}{x} \| \frac{u_i}{\sum_{j \in S} u_j}\right)\right), \\
    \Prob{ Y \geq d} &\leq \exp\left( -x KL\left(\frac{x-d}{x} \| \frac{\sum_{j \in S\backslash\{i\}} u_j}{\sum_{j \in S} u_j}\right)\right), \\
    \Prob{ Y \leq d} &\geq \frac{1}{\sqrt{8d(1-\frac{d}{x})}} \exp\left( -x KL\left(\frac{d}{x} \| \frac{u_i}{\sum_{j \in S} u_j}\right)\right), \\
    \Prob{ Y \geq d} &\geq \frac{1}{\sqrt{8(x-d)(1-\frac{x-d}{x})}}‚ \exp\left( -x KL\left(\frac{x-d}{x} \| \frac{\sum_{j \in S\backslash\{i\}} u_j}{\sum_{j \in S} u_j}\right)\right).
\end{align*}
Replacing the bounds in the probability that arm $i$ is identified as the winner in $x$ duels gives us the stated result.
\end{proof}

\end{document}